\documentclass{article}
\usepackage{spconf,amsmath,graphicx}
\usepackage{amsfonts} \usepackage{multirow} \usepackage{xcolor} \usepackage{xparse}
\usepackage{multicol, blindtext}
\usepackage{color}
\usepackage{amsmath}
\usepackage{algorithm}
\usepackage[noend]{algorithmic} 
\def\x{{\mathbf x}}

\def\R{\mathbb{R}}

\def\D{\mathbf{D}}
\def\X{\mathbf{X}}
\def\x{\mathbf{x}}
\def\A{\mathbf{A}}

\def\d{\mathbf{d}}

\def\a{\mathbf{a}}

\newtheorem{theorem}{Theorem}
\newtheorem{lemma}[theorem]{Lemma}

\newenvironment{proof}[1][Proof]{\begin{trivlist}
\item[\hskip \labelsep {\bfseries #1}]}{\end{trivlist}}

\NewDocumentCommand{\ceil}{s O{} m}{  \IfBooleanTF{#1}     {\left\lceil#3\right\rceil}     {#2\lceil#3#2\rceil} }

\title{Concave Losses for Robust Dictionary Learning }
\twoauthors
 {Rafael Will M. de Araujo, R. Hirata Jr \sthanks{These authors thank CAPES (\#88881.135686/2016-01) and FAPESP (\# 2015/01587-0).}}
	{University of S\~{a}o Paulo\\
	Institute of Mathematics and Statistics\\
	\ninept{Rua do Mat\~{a}o, 1010 -- 05508-090 -- S\~{a}o Paulo-SP, Brazil}}
 {Alain Rakotomamonjy \sthanks{AR acknowledges funding
from the R\'egion Normandie, European funding through the GRR DAISI, the FEDER DAISI.}}
	{Universit\'{e} de Rouen Normandie\\
	LITIS EA 4108\\
	\ninept{76800 Saint-\'{E}tienne-du-Rouvray, France}}
\begin{document}
\maketitle
\begin{abstract}
  Traditional dictionary learning methods are based on quadratic
  convex loss function and thus are sensitive to outliers.  In this
  paper, we propose a generic framework for robust dictionary learning
  based on concave losses. We provide results on composition of concave
  functions, notably regarding supergradient computations, that are
key for developing generic dictionary learning algorithms applicable to smooth and non-smooth
losses. In order to improve identification of outliers, we introduce  an 
initialization heuristic based on
  undercomplete dictionary learning. Experimental results using synthetic and real data   demonstrate that our method is able to better detect
outliers, is capable of generating better
  dictionaries, outperforming state-of-the-art methods such as K-SVD
  and LC-KSVD.
\end{abstract}
\section{Introduction}
\label{sec:introduction}
Dictionary Learning is 
an important and widely used tool in Signal
Processing and Computer Vision. Its versatility 
is well acknowledged and 
it can be applied for denoising or for representation learning
prior to classification  \cite{aharon2006rm,Mairal:2009-online-dl}. 
The method consists in learning a
set of overcomplete elements (or atoms) 
which are 
useful for describing examples at hand. 
In this context, each example is represented as a potentially sparse linear span of the atoms. 
Formally, given a data matrix composed of $n$ elements of dimension $d$, $\X \in \R^{d \times n}$ and each column being an example $\x_i$,
 the dictionary learning problem is given by:
\begin{equation}\label{eq:dl}
\min_{\D  \in \R^{d \times k} , \A \R^{k \times n}} \frac{1}{2} \sum_{i=1}^n \|\x_i - \D \a_i\|_2^2 + \Omega_D(\D) + \Omega_A(\A)
\end{equation} 
where $\Omega_D$ and $\Omega_A$ represent some constraints and/or
penalties on the dictionary set $\D$
and the matrix coefficient $\A $,
each column being a linear combination coefficients $\a_i$
so that $ \x_i \approx \D \a_i$.
 Typical
regularizers are sparsity-inducing penalty on $\A$, or unit-norm
constraint on each dictionary element although a wide variety of
penalties can be useful \cite{Tibshirani:1996-lasso,
  Bach:2012-regularizers, rakotomamonjy2013applying}.

As depicted by the mathematical formulation of the problem, the learned dictionary $\D$  depends on training examples $\{\x_i\}_{i=1}^n$. However, because of the quadratic loss function in the data fitting
term, $\D$ is in addition, very sensitive to outlier examples. Our goal here is to address the robustness of the 
approach to outliers.  For this purpose, we  consider
loss functions that downweight the importance of outliers in
$\X$ making the learned dictionary less sensitive to them.

Typical approaches in the literature, that aim at mitigating influence of outliers, use Frobenius norm or component-wise
$\ell_1$ norm as data-fitting term instead of the squared-Frobenius one
\cite{nie2010efficient,wang2016fast}. 
Some works propose  loss functions such as the $\ell_q$ function, with 
$q \leq 1$ function or the capped  function $g(u) = \min(u,\epsilon)$,
for $u > 0$ \cite{wang2013semi,jiang2015robust}. 
Due to these non-smooth and non-convex loss function, the resulting dictionary
learning problem is more difficult to solve than the original one
given in Equation (\ref{eq:dl}). As such, authors have developed
algorithms based on a iterative reweighted least-square approaches
tailored to the loss function $\ell_q$ or $\min(u,\epsilon)$
\cite{wang2013semi,jiang2015robust}.

Our contribution in this paper is: (i) to introduce a generic framework for robust dictionary learning by considering as loss function the composition of the Frobenius norm and some concave loss functions (our framework encompasses previously
proposed methods while enlarging the set  of applicable loss functions); (ii) to propose a generic majorization-minimization algorithm applicable to concave, smooth or non-smooth  loss functions. Furthermore, 
because the resulting learning problem is
non-convex, its solution is sensitive to initial conditions, hence we  propose a novel heuristic for dictionary initialization that helps in detecting
outliers more efficiently during the learning process.

\section{Concave Robust Dictionary Learning}
\label{sec:rdl}

\subsection{Framework and Algorithm}
\label{ssec:our_robust_dl}

In order to robustify the dictionary learning process against outliers,
we need a learning problem that puts less emphasis on examples
that are not ``correctly'' approximated by the learned dictionary.
Hence, we propose the following generic learning problem:
\begin{equation}
  \label{eq:rdl}
  \min_{\D,\A}  \frac{1}{2} \sum_i F(\| \x_i - \D \a_i\|_2^2) + \Omega_D(\D) + \Omega_A(\A). 
\end{equation}
where $F(\bullet)$ is a function over $\R_{>0}$.
  Note that in
the sequel, we will not focus on the penalty and constraints over the
dictionary elements and coefficients $\A$. Hence, we  consider
them as the classical unit-norm constraint over $\d_j$ and the
$\ell_1$ sparsity-inducing penalty over $\{\a_i\}$.

Concavity of $F$ is crucial for robustness as it helps
in down-weighting influence of large  $\|\x_i - \D\a_i\|_2$.
 For instance, if we
set $F(\bullet)= \sqrt{\bullet}$, the above problem is similar to the convex robust
dictionary learning proposed by Wang et al. \cite{wang2016fast}. In order to provide 
better robustness, our goal is to introduce a generic
form of $F$ that leads to a concave loss with respect to $\|\x_i - \D\a_i\|_2$. instead of a linear, yet concave 
one as in \cite{wang2016fast}.

In this work, we emphasize
robustness by considering $F$ as the composition of two concave
functions $F(\bullet) = g(\bullet) \circ \sqrt{\bullet} $, with $g$ a non-decreasing
concave function over $\R_{>0}$, such as those  used for sparsity-inducing
penalties. Typically, $g(\bullet)$ can be the $q-$power, $q \leq 1$ function
$u^q$, the log function $\log(\epsilon + {u})$, the SCAD function
\cite{Fan:2001-scad}, or the capped-$\ell_1$ function
$\min(u, \epsilon)$, or the MCP function \cite{zhang2010nearly}. A key property
on $F$ is that concavity is preserved by the composition
of some specific concave functions as proved by the following lemma
which proof is omitted for space reasons.
\begin{lemma} Let $g$ be a non-decreasing concave function on 
$\R_{>0}$ and $h$ be a concave function on a domain $\Omega$ to $\R_{>0}$,
then $ g \circ h$ is concave. Furthermore, if $g$ is an increasing function
then $ g \circ h$ is strictly concave. 
\end{lemma}
In our framework, $h$ is the square-root function with
$\Omega = \R_{>0}$. 
In addition, functions $g$, such as those given above, are either a concave or strictly concave functions and are all non-decreasing, hence $F = g \circ h$ is concave. Owing to concavity,  for any $u_0$ and $u$ in
$\R_{>0}$,
$$
F(u) \leq F(u_0) + F'(u_0) (u - u_0)
$$
where $F'(u_0)$ is an element of the superdifferential of $F$ at $u_0$. As $F$ is concave, the superdifferential is always non-empty and if $F$ is smooth at $u_0$, then $F'(u_0)$ is simply the gradient of $F$ at $u_0$. 
However, since $F$ is a composition of functions, in a non-smooth case, computing superdifferential is difficult unless the inner function is a linear function \cite{rockafellar2015convex}. Next lemma provides a key result showing that a  supergradient of $g \circ \sqrt{\bullet}$
can be simply computed using chain rule because $\sqrt{\bullet}$ is a
bijective function on $\R_{>0}$ to $\R_{>0}$ and $g$ is non-decreasing.  

\begin{lemma} Let $g$ a non-decreasing concave function on 
$\R_{>0}$ and $h$ a bijective differentiable concave function on a domain $\R_{>0}$ to $\R_{>0}$,
 then  if $g_1$ is a supergradient of
$g$ at $z$ then $g_1 \cdot h'(s) $  is a supergradient of $g \circ h $ at a point $s$ so
that $z = h(s)$.
\end{lemma}
\begin{proof} As $g_1 \in \partial g(z)$, we have $\forall y, g(y) \leq
g(z) + g_1 \cdot ( y -z)$.  Owing to bijectivity of $h$, define $t$ and $s$ so that $y
= h(t)$  and $z=h(s)$. In addition, concavity of $h$ gives $h(t) - h(s) \leq
h'(s) (t-s)$ and because $g$ is non-decreasing, $g_1 \geq 0$. 
Combining everything, we have $g_1 \cdot ( y -z) =
g_1 \cdot (h(t) - h(s)) \leq g_1 h'(s) (t-s) $. Thus
 $\forall t, g(h(t)) \leq g(h(s)) +  g_1 h'(s) (t-s)$ which concludes
the proof since $g_1$ is a supergradient of $g$ at $h(s)$.
\end{proof}

\begin{algorithm}[t]
\footnotesize \caption{The proposed Robust DL method}
\label{alg:proposed_robust_dl}
\begin{algorithmic}[1] 	\REQUIRE Data matrix $\mathbf{X} \in \mathbb{R}^{d \times n}$, dictionary size $k$, $\lambda$, $\epsilon$, $M$.
  		\IF{($k > d$) and (\textit{use undercomplete initialization})}
    	\STATE Initialize $\D$ and $s$ with Algorithm~\ref{alg:undercomplete_initialization}
    \ELSE
    	\STATE random initialization of $\D$, $\A$ 
        \STATE $s_j = 1$ for $j=1, \ldots, n$
    \ENDIF
    \FOR{$i=1$ to $M$}
    	\REPEAT
        	\STATE Update $\D$ by solving Equation \ref{eq:dicoupdate}
            \FOR{$j=1$ to $n$}
           		\STATE $\a_j \gets \frac{1}{2} ||\mathbf{x}_j - \D \a||_2^2 + \frac{\lambda}{s_j} ||\a||_1$
            \ENDFOR
        \UNTIL{convergence}
        \FOR{$j=1$ to $n$}
    		       		            \STATE update $s_j$ according to Equation \ref{eq:upds}
        \ENDFOR
    \ENDFOR
	\ENSURE $\D$, $s$
\end{algorithmic} \end{algorithm}

Based on the above majorizing linear function property of concave
functions and because  in our case $F'(u_0)$ can easily be
computed, we consider a majorization-minimization approach for
solving Problem (\ref{eq:rdl}).  
Our iterative algorithm consists, at iteration $\kappa$, in
approximating the concave loss function $F$ at the current solution
$\D_{\kappa}$ and $\A_{\kappa}$ and then solve the resulting
approximate problem for $\D$ and $\A$. This yields in solving:
\begin{equation}
  \label{eq:iteraterdl}
    \min_{\D,\A}  \frac{1}{2} \sum_i  s_i \| \x_i - \D \a_i\|_2^2 + \Omega_D(\D) +
\Omega_A(\A)
\end{equation}
where $s_i = [g \circ \sqrt{\bullet}~]^\prime$ at
$\D_{\kappa}$ and $\a_{\kappa,i}$. Since, we have
$$
[g \circ \sqrt{\bullet}~]^\prime(u_0) = \frac{1}{2 \sqrt{u_0}}g^{\prime}(
\sqrt{u_0}) 
$$
weights $s_i$ can be defined as 
\begin{equation}\label{eq:upds}
s_i = \frac{g'(\| \x_i -\D_{\kappa} \a_{\kappa,i}\|_2)}{2 \| \x_i
  -\D_{\kappa} \a_{\kappa,i}\|_2}.
\end{equation}

This definition of  $s_i$ can be nicely interpreted. Indeed, if
$g$ is so that  
$\frac{g'(u)}{u}$ becomes small as $u$ increases, examples with
large residual values $\| \x_i   -\D_{\kappa} \a_{\kappa,i}\|_2$ have
less importance in the learning problem (\ref{eq:iteraterdl}) because
their corresponding values $s_i$ are small.

Note how the composition $g \circ \sqrt{\bullet}$ allows us to write
the data fitting term with respect to the squared residual norm so that at
each iteration, the problem (\ref{eq:iteraterdl}) to solve
is simply a weighted
 smooth dictionary learning problem, convex in each of its parameters,
that can be addressed using off-the-shelf tools.
As
such, it can be solved alternatively for $\D$ with fixed $\A$ and then
for $\A$ with fixed $\D$. For fixed $\A$, the optimization problem is
thus:
\begin{equation}\label{eq:dicoupdate}
\min_\D \frac{1}{2} \sum_i \| \tilde \x_i - \D \tilde \a_i\|_2^2 + \Omega_D(\D)
\end{equation}
where $\tilde \x_i = \sqrt{s_i} \x_i$ and
$\tilde \a_i = \sqrt{s_i} \a_i$. This problem can be solved using a
proximal gradient algorithm or block-coordinate descent algorithm as
given in Mairal et al. \cite{Mairal:2009-online-dl}. For fixed $\D$,
the problem is separable in $\a_i$ and each sub-problem is equivalent
to a Lasso problem with regularization $\frac{\lambda}{s_i}$.

The above algorithm is generic in the sense that it is applicable to any continuous concave and non-decreasing function $g$, even non-smooth ones. This is in 
constrast with algorithms proposed in \cite{wang2013semi,jiang2015robust} which have been tailored to
some specific functions $g$. In addition, the convergence in objective value of the algorithm is guaranteed for any of these $g$ functions, by the fact that the objective value  in Equation \ref{eq:rdl} decreases at each iteration while it is obviously lower bounded.

\subsection{Heuristic for initialization}

The problem we are solving is a non-convex problem and its solution is
thus very sensitive to initialization. The presence of outliers in the
data matrix $\X$ magnifies this sensitivity and increases the need for
a proper initialization of $s_i$ in our iterative algorithm based
on Equation (\ref{eq:iteraterdl}). If we were  able to identify
outliers before learning, then we would assign $s_i=0$ to these
samples so that they become irrelevant for the dictionary learning
problem. However, detecting outliers in a set of samples is a
difficult learning problem by itself \cite{Chandola:2007-outlier}.

Our initialization heuristic is based on the intuition that if most
examples belong to a linear subspace of $\R^d$ while outliers leave
outside this subspace, then these outliers can be better identified by
using an undercomplete dictionary learning than an overcomplete
one. Indeed, if the sparsity penalty is weak enough, then an
overcomplete dictionary can approximate well any example leading to a large value of $s_i$ even the
for outliers.

Hence, if the number of dictionary to learn is larger than the
dimension of the problem, we propose to initialize $\D$ and $s$ by
learning mini-batches of size $b<d$ of dictionary atoms using one iteration of
Alg.~\ref{alg:proposed_robust_dl} initialized with
$s_i = 1, \forall i \in [1, \ldots, n]$, a random dictionary and random
weigths $\A$. If
we have only a small proportions of outliers, we make the hypothesis
that the learning problem will focus on dictionary atoms that span a
subspace that better approximates non-outlier samples. Then, 
as each set of learned mini-batch dictionary atoms leads to a different
error $\| \x_i - \D \a_i\|_2$ and thus to different $s_i$ as defined
in Equation $3$, we estimate $s_i$ as the average $s_i$ over the
number of mini-batch and we expect $s_i$ to be small if  $i$-th example
is an outlier.
This initialization strategy is presented in Alg.~\ref{alg:undercomplete_initialization}.

\begin{algorithm}[t]
\footnotesize \caption{Undercomplete initialization}
\label{alg:undercomplete_initialization}
\begin{algorithmic}[1] 	\REQUIRE Data matrix $\mathbf X$, dictionary $\D \in \R^{d \times k}$, with $d<k$,  number of atoms in each batch $b < k $, parameters $\lambda$ and $\epsilon$.
	
    	\STATE $N \gets \ceil[\big]{\frac{k}{b}}$   \COMMENT{number of batches}
        \STATE $s = 0$
        \STATE Initialize $\D = [\d_1, \ldots, \d_k]$ as a zero matrix
        \FOR{$i=0$ to $(N-1)$}
            \STATE I = indices related to $i$-th batch
            \STATE $\hat \D, \hat s \gets \textrm{Algorithm~\ref{alg:proposed_robust_dl}}\big( \X, |I|, \lambda, \epsilon, 1 \big)$
            \STATE $\D_I \gets \hat \D$ \COMMENT{assign learned dictionary to the appropriate indices}
            \STATE $s \gets s + \hat s$ \COMMENT{accumulate weights}
        \ENDFOR
        \STATE $s \gets \frac{s}{N}$ \COMMENT{compute average}
	\ENSURE $\D$, $s$
\end{algorithmic} \end{algorithm}

\section{Experiments}
\label{sec:experiments}

\subsection{Experiments on synthetic data}
\label{ssec:experiments_synthetic_data}

\begin{figure}[t]
	\centering
             ~\hfill
         \includegraphics[width=0.22\textwidth]{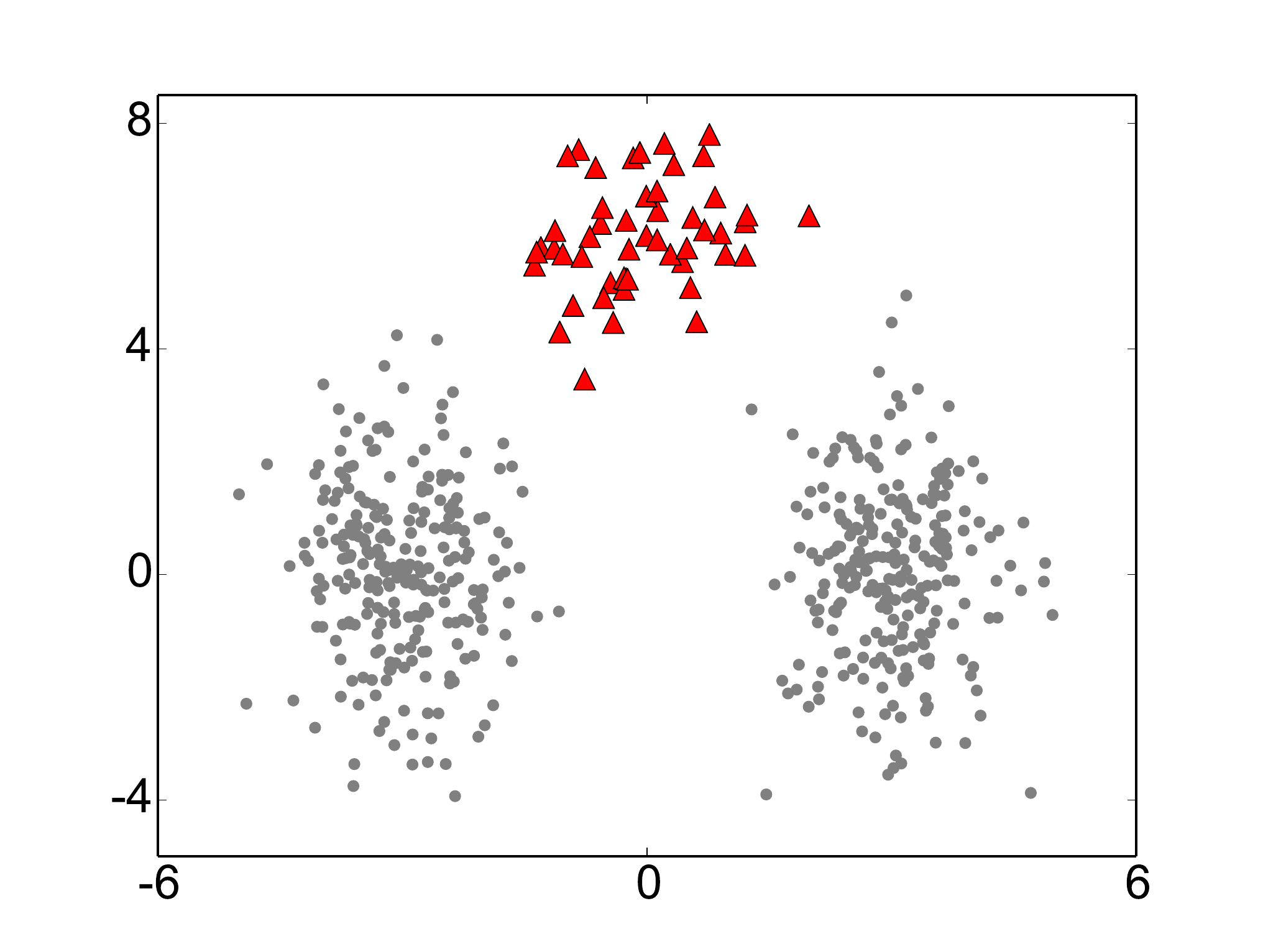}~\hfill
         \includegraphics[width=0.22\textwidth]{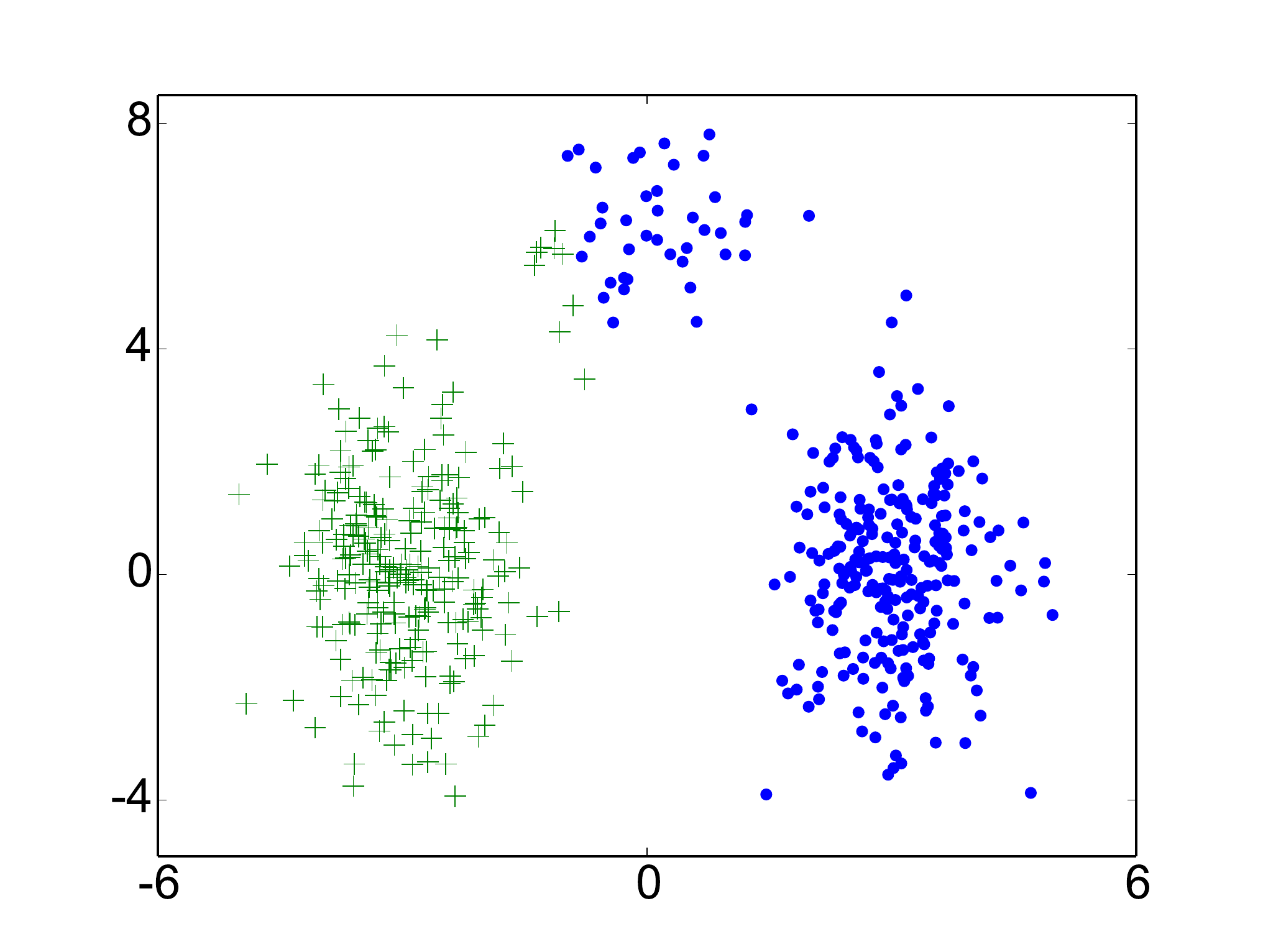}\hfill~\\
         ~\hfill
        \includegraphics[width=0.22\textwidth]{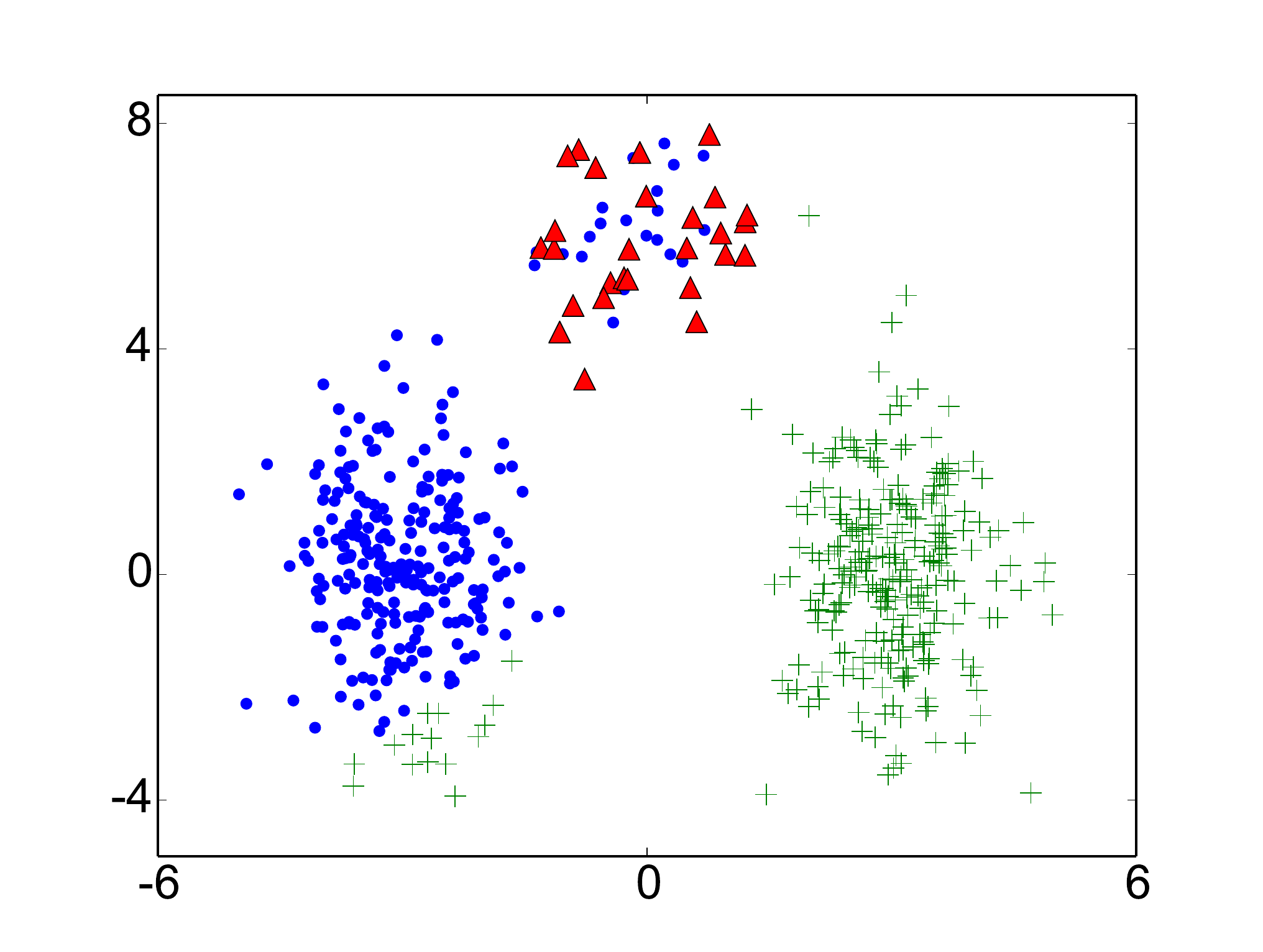}~\hfill~
        \includegraphics[width=0.22\textwidth]{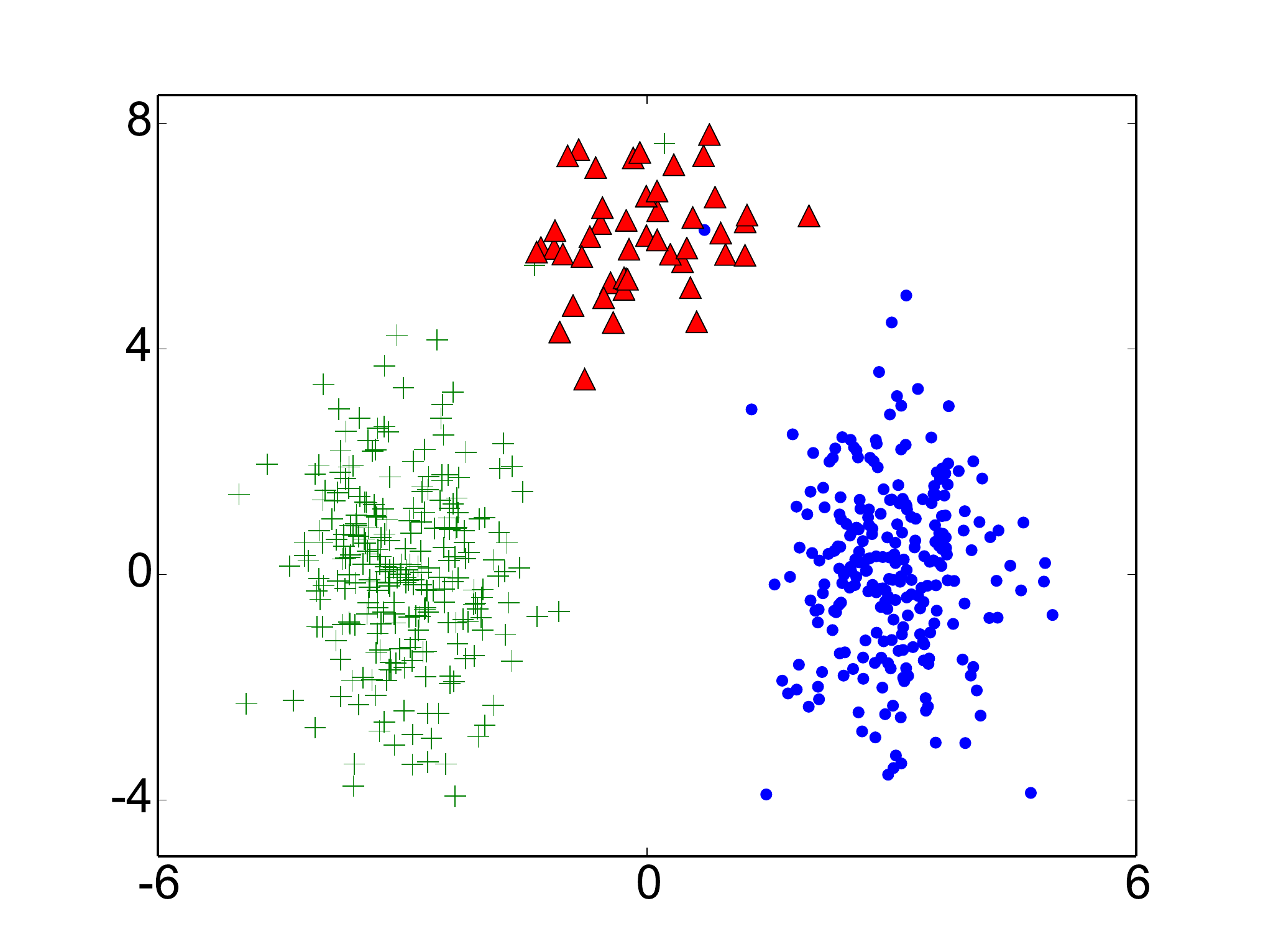}\hfill~\\
            \caption{Synthetic 2D data drawn from two Gaussian distributions. The outliers are represented as the red triangles.
            (top-left) original data with outliers (top-right) Clustering with K-SVD
            (bottom-left) Clustering with proposed method with $g(u)= u$. (bottom-right) Clustering with proposed method using the log function.}
    \label{fig:2d_experiments}
\end{figure}

\begin{figure*}[h]
	    \hspace{-4mm}
    \begin{minipage}{0.33\linewidth}
        \centering
        \includegraphics[width=1.0\textwidth]{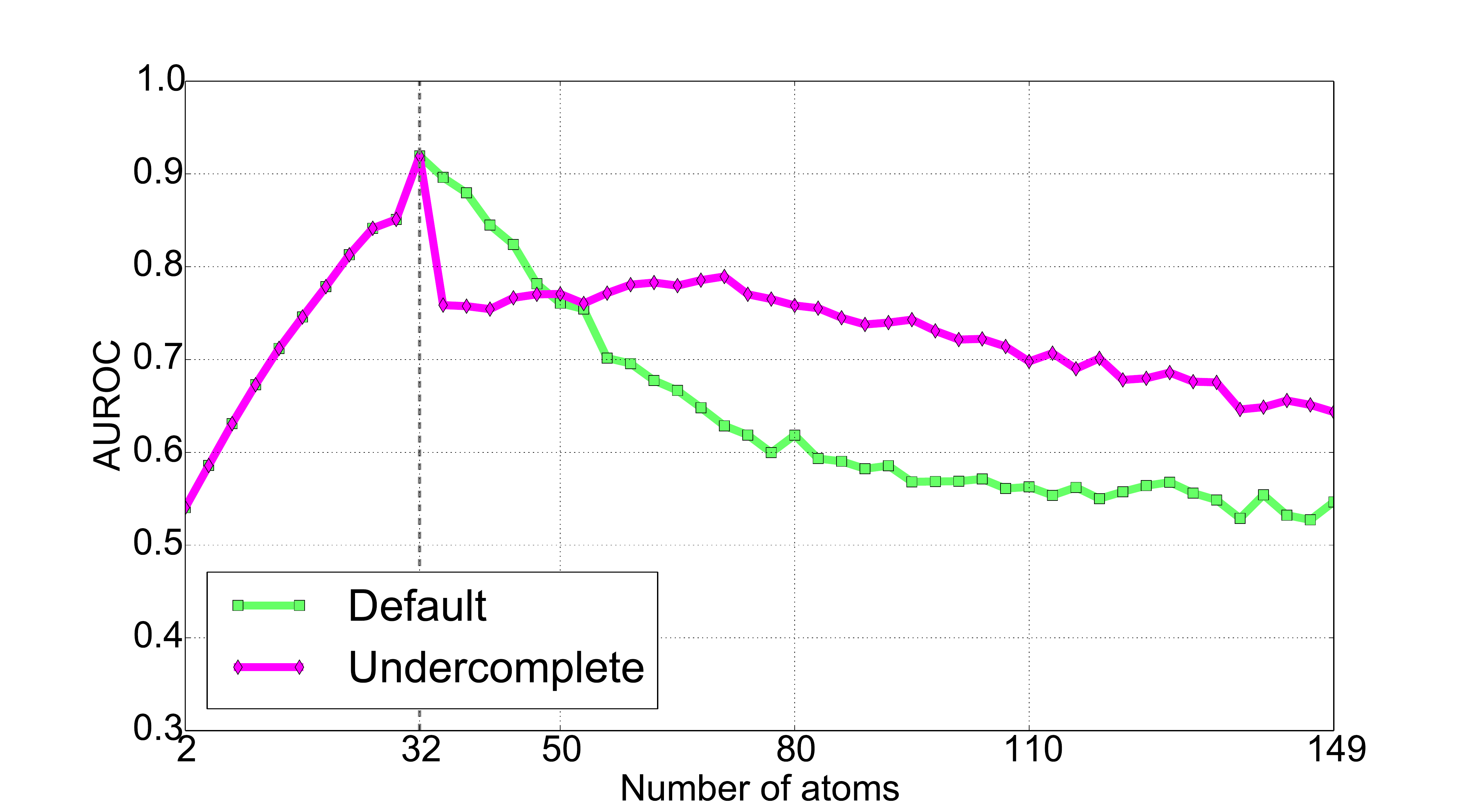}\\
        (a) Different dictionary sizes, with 1000 samples and 10\% are outliers.
    \end{minipage}
    \begin{minipage}{0.33\linewidth}
        \centering
        \includegraphics[width=1.0\textwidth]{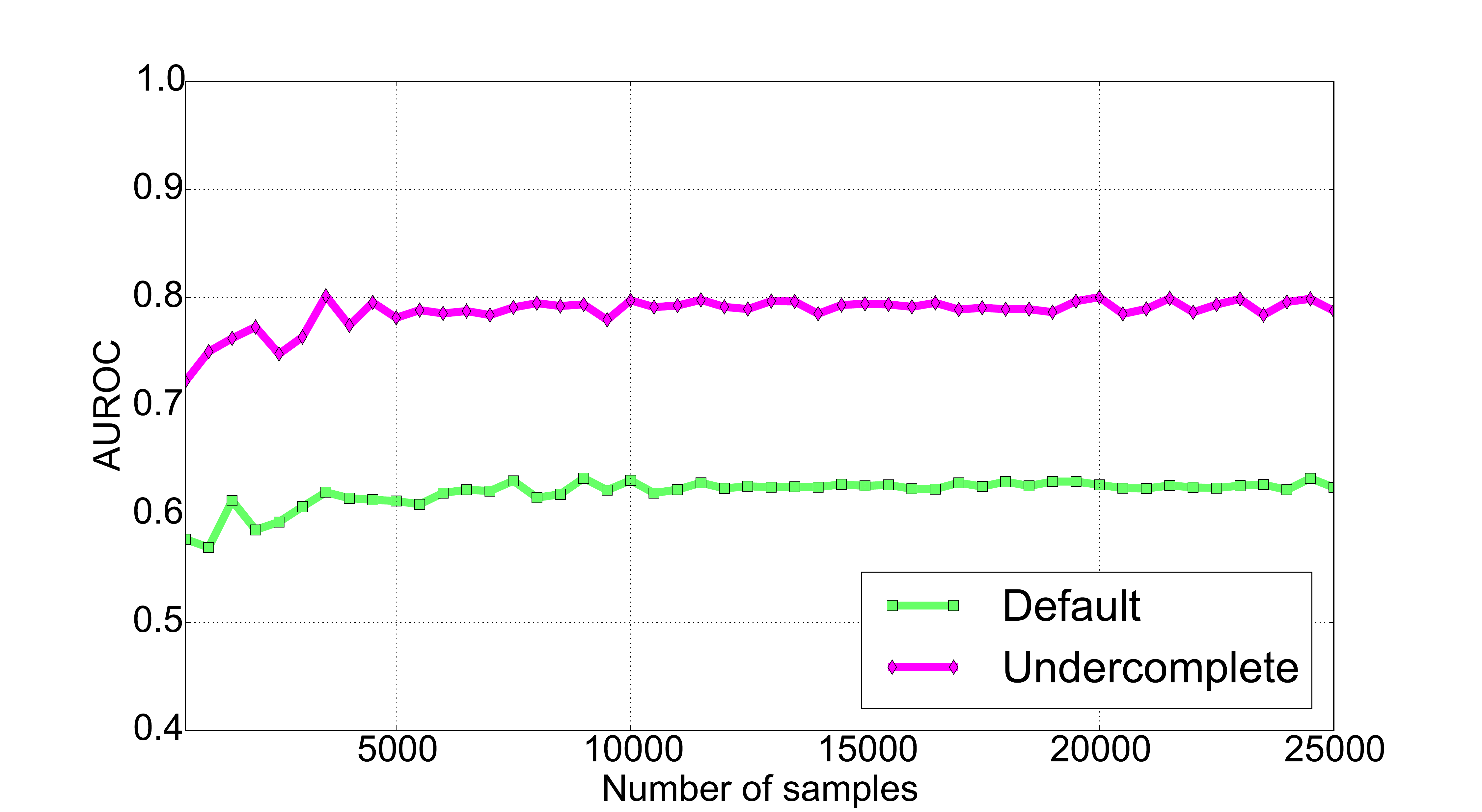}\\
        (b) Different number of samples, where 10\% are outliers.
    \end{minipage}
    \begin{minipage}{0.33\linewidth}
        \centering
        \includegraphics[width=1.0\textwidth]{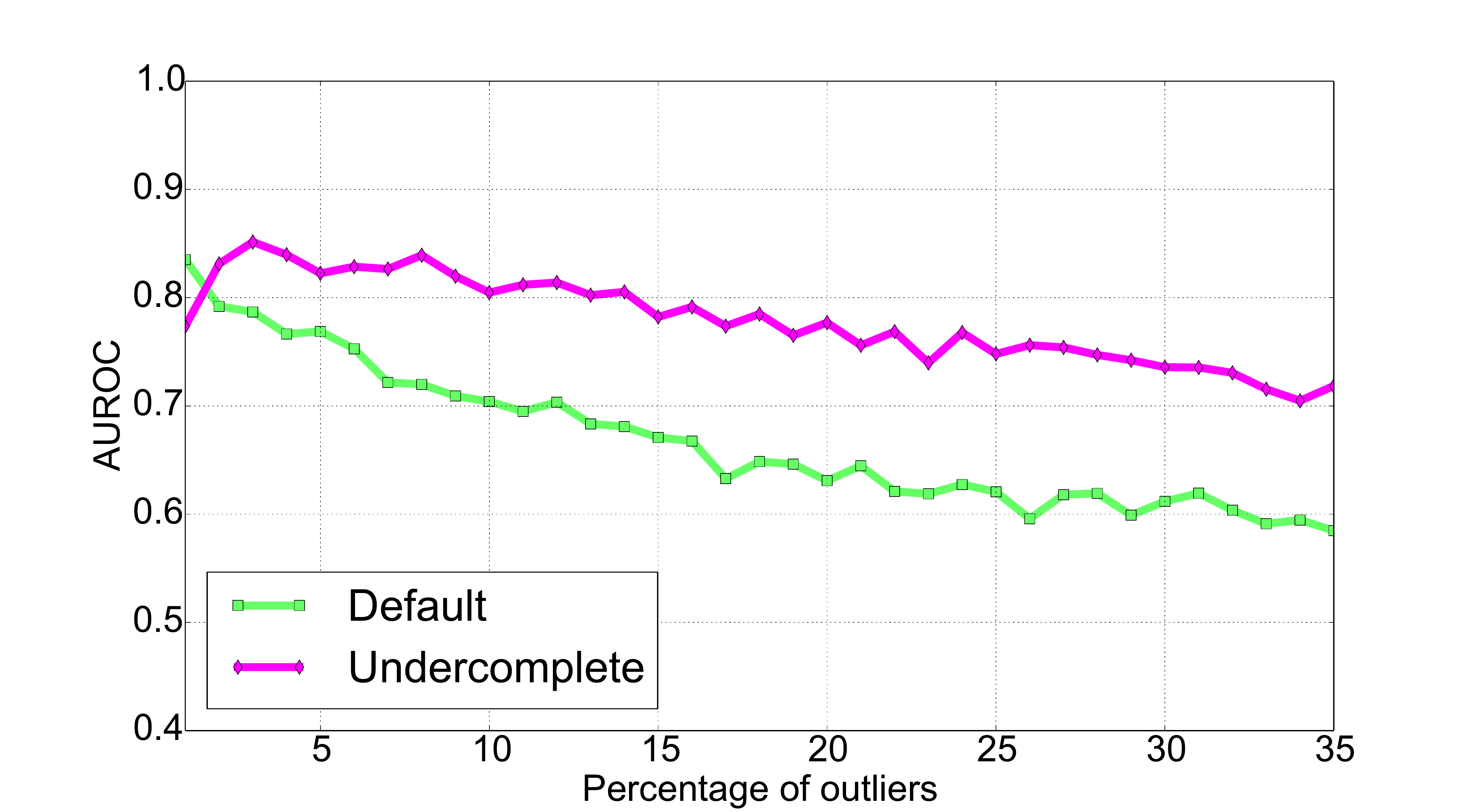}\\
        (c) Different outlier ratios (\%), with 1000 samples and 64 atoms.
    \end{minipage}
    \caption{Performance of the proposed method with multidimensional data.}
    \label{fig:multidimensional_experiments}
\end{figure*}

We use synthetic generated datasets with outliers to demonstrate that our method is robust against outliers. Figure~\ref{fig:2d_experiments} presents two clusters generated from two Gaussian distributions, each containing 250 points along with 50 outliers represented as the red triangles, far away from the clusters. Figure~\ref{fig:2d_experiments} also shows the clustering results using K-SVD \cite{Aharon:2006-ksvd} as well as the proposed method when $g(u) = u$ and the $log(\epsilon + u)$ functions, respectively. Then, we compare how many of the original outliers are among the 50 highest reconstruction values. The proposed method using the log function proved to be the most robust against outliers, with 47 from the 50 true outliers detected. It is followed by the variant with the identity function, which identified 27 outliers, and finally by K-SVD, which was naturally not able to identify any of the original outliers. This example also shows that concavity of $g$ helps in better identifying outliers.

To further evaluate our proposal, we performed experiments with higher dimensional data (fixed at 32 dimensions). To generate the data, we use the approach described by Lu et al. \cite{Lu:2013-online-dl} to create synthetic data based on a dictionary and sparse coefficients. The metric adopted to compare the results is the AUC Curve (AUROC) of outlier scores $\{s_i\}$ after running Alg.~\ref{alg:proposed_robust_dl}: outliers should have scores $1/s_i$ larger than non-outliers, and each point is the average of 5 runs using newly generated data. In Fig.~\ref{fig:multidimensional_experiments}a one can observe that the behavior for both lines is the same until the number of atoms reach 32, since $k \leq d$ and the condition in the first line of Alg.~\ref{alg:undercomplete_initialization} is not met. The performance of the undercomplete initialization method also deteriorates for dictionary sizes a little bit greater than $d$, but as far as $k$ starts to increase it becomes evident that this method outperforms the default initialization. Figure~\ref{fig:multidimensional_experiments}b shows that our method stays very stable independent of the number of samples, given a constant outlier ratio, regardless of the initialization method. Finally, Fig.~\ref{fig:multidimensional_experiments}c shows the behavior of both initialization strategies in scenarios where the outlier proportion changes. It can be noticed that the AUROC values decrease slowly as long as the number of outliers in the samples increase. This is natural since when proportion of outliers is large, outliers can hardly be considered outliers anymore.

\subsection{Human attribute classification}
\label{ssec:experiments_real_data}

In order to prove that our robust dictionary learning method is really beneficial to real data contexts, we also evaluate its performance on the MORPH-II dataset \cite{Ricanek:2006-morph}, one of the largest labeled face databases available for gender and ethnicity classification, containing more than 40,000 images. Before the training and classification take place, the images are preprocessed, which consists of face detection, align the image based on the eye centers, as well as cropping and resizing. Finally, they are converted to grayscale and SIFT \cite{Lowe:1999-sift} descriptors are computed from a dense grid.

The experiments are run with the proposed method using both the default and the undercomplete initialization approaches using the log function, and then compared with state-of-the-art methods such as K-SVD and LC-KSVD \cite{Jiang:2011-lcksvd}. The classifier uses a Bag of Visual Words (BoVW) approach \cite{Csurka:2004-bag} by replacing the original K-Means algorithm with each of those methods, and then generating a image signature (histogram of frequencies) using the computed clusters, which are later fed to a SVM. This SVM uses a RBF (Radial Basis Function) kernel with tuned $\gamma$ and $C$ parameters. The number of atoms is set to 200 for all experiments.

\begin{table}[h]
\centering
\scalebox{0.70}{
\begin{tabular}{|c||c|c|}
\hline
\multirow{2}{*}{\textit{Method}} & \textit{Ethnicity} & \textit{Gender} \\
 & \textit{accuracy} & \textit{accuracy} \\ \hline \hline
\textbf{Our RDL (default)}   & \textbf{96.28} & \textbf{84.76} \\
\hline
\textbf{Our RDL (undercomplete)}   & \textbf{96.90} & \textbf{85.79} \\
\hline
K-SVD   & 96.23 & 81.88 \\
\hline
LC-KSVD1   & 96.24 & 83.91 \\
\hline
LC-KSVD2   & 95.69 & 84.69 \\
\hline
\end{tabular}
} \caption{Average classification accuracies (\%) for ethnicity and gender classification on the MORPH-II dataset.}
\label{tab:results_ethnicity_gender}
\end{table}

Each experiment is the average of 3 runs, each one using 300 selected images per class for training, and the remaining images for classification. The total number of images per class is as follows: 32,874 Africans plus 7,942 Caucasians for ethnicity classification, and 6,799 Females plus 34,017 Males for gender classification. Table~\ref{tab:results_ethnicity_gender} shows the overall accuracies. These experiments clearly demonstrate that the quality of the dictionaries computed by the proposed robust dictionary learning method is indeed superior
even to methods that uses labels for dictionary learning  \cite{Jiang:2011-lcksvd}.

\section{Conclusions}
\label{sec:conclusions}

In this work, we proposed a generic dictionary learning framework which takes advantage of a composition of two concave functions to generate robust dictionaries with very little outlier interference. We also came up with a heuristic initialization which can further increase the identification of outliers, through the use of undercomplete dictionaries. Experiments on synthetic and real world datasets show that the proposed methods outperform some of the state-of-the-art methods such as K-SVD and LC-KSVD, since our approaches are able to achieve higher quality dictionaries which better generalize data.

\vfill\pagebreak

\bibliographystyle{IEEEbib}

\end{document}